\documentclass{article}

\usepackage{microtype}
\usepackage{graphicx}
\usepackage{subfigure}
\usepackage{booktabs} 

\usepackage{hyperref}

\usepackage[accepted]{icml2025}
\usepackage{amsmath}
\usepackage{amssymb}
\usepackage{mathtools}
\usepackage{amsthm,enumitem}

\usepackage[capitalize,noabbrev]{cleveref}

\theoremstyle{plain}
\newtheorem{theorem}{Theorem}

\newtheorem{lemma}{Lemma}

\theoremstyle{definition}
\newtheorem{definition}{Definition}
\newtheorem{assumption}{Assumption}
\theoremstyle{remark}
\newtheorem{remark}{Remark}

\usepackage[textsize=tiny]{todonotes}

\icmltitlerunning{Quantum Speedups in Regret Analysis of Infinite Horizon Average-Reward
MDPs}

\begin{document}

\twocolumn[
\icmltitle{Quantum Speedups in Regret Analysis of Infinite Horizon \\ Average-Reward
Markov Decision Processes}



\icmlsetsymbol{equal}{*}

\begin{icmlauthorlist}
\icmlauthor{Bhargav Ganguly}{equal,PurdueIE}
\icmlauthor{Yang Xu}{equal,PurdueIE}
\icmlauthor{Vaneet Aggarwal}{PurdueIE}
\end{icmlauthorlist}

\icmlaffiliation{PurdueIE}{Purdue University, West Lafayette, IN, USA}
\icmlcorrespondingauthor{Bhargav Ganguly}{bganguly@purdue.edu}
\icmlcorrespondingauthor{Yang Xu}{xu1720@purdue.edu}
\icmlcorrespondingauthor{Vaneet Aggarwal}{vaneet@purdue.edu}

\icmlkeywords{Machine Learning, ICML}

\vskip 0.3in
]



\printAffiliationsAndNotice{\icmlEqualContribution} 

\begin{abstract}

This paper investigates the potential of quantum acceleration in addressing infinite horizon Markov Decision Processes (MDPs) to enhance average reward outcomes. We introduce an innovative quantum framework for the agent's engagement with an unknown MDP, extending the conventional interaction paradigm. Our approach involves the design of an optimism-driven tabular Reinforcement Learning algorithm that harnesses quantum signals acquired by the agent through efficient quantum mean estimation techniques. Through thorough theoretical analysis, we demonstrate that the quantum advantage in mean estimation leads to exponential advancements in regret guarantees for infinite horizon Reinforcement Learning. Specifically, the proposed Quantum algorithm achieves a regret bound of $\tilde{\mathcal{O}}(1)$\footnote{$\tilde{\mathcal{O}}(\cdot)$ conceals logarithmic terms of $T$.}, a significant improvement over the $\tilde{\mathcal{O}}(\sqrt{T})$ bound exhibited by classical counterparts.

\end{abstract}
\section{Introduction}
Quantum Machine Learning (QML) has garnered remarkable interest in contemporary research, predominantly attributed to the pronounced speedups achievable with quantum computers as opposed to their classical analogs \cite{biamonte2017quantum,bouland2023quantum}. The fundamental edge of quantum computing arises from the unique nature of its fundamental computing element, termed a qubit, which can exist simultaneously in states of 0 and 1, unlike classical bits that are restricted to either 0 or 1. This inherent distinction underpins the exponential advancements that quantum computers bring to specific computational tasks, surpassing the capabilities of classical computers.

In the realm of Reinforcement Learning (RL), an agent embarks on the task of finding an efficient policy for Markov Decision Process (MDP) environment through repetitive interactions \cite{sutton2018reinforcement}. RL has found notable success in diverse applications, including but not limited to autonomous driving, ridesharing platforms, online ad recommendation systems, and proficient gameplay agents \cite{silver2017mastering,al2019deeppool,chen2021deepfreight,bonjour2022decision}. Most of these setups require decision making over infinite horizon, with an objective of average rewards. This setup has been studied in classical reinforcement learning \cite{auer2008near,agarwal2021concave,fruit2018efficient}, where $\tilde{\mathcal{O}}(\sqrt{T})$ regret across $T$ rounds has been found {meeting the theoretical lower bound \cite{jaksch2010near}}. This paper embarks on an inquiry into the potential utilization of quantum statistical estimation techniques to augment the theoretical convergence speeds of tabular RL algorithms within the context of infinite horizon learning settings.

A diverse array of quantum statistical estimation primitives has emerged, showcasing significant enhancements in convergence speeds and gaining traction within Quantum Machine Learning (QML) frameworks \cite{brassard2002quantum,harrow2009quantum,gilyen2019quantum}. Notably, \cite{hamoudi2021quantum} introduces a quantum mean estimation algorithm that yields a quadratic acceleration in sample complexity when contrasted with classical mean estimation techniques. In this work, we emphasize the adoption of this particular mean estimation method as a pivotal component of our methodology. It strategically refines the signals garnered by the RL agent during its interaction with the enigmatic quantum-classical hybrid MDP environment.

It is pertinent to highlight a fundamental element in the analysis of conventional Reinforcement Learning (RL), which involves the utilization of  martingale convergence theorems. These theorems play a crucial role in delineating the inherent stochastic process governing the evolution of states within the MDP. Conversely, this essential aspect lacks parallelism within the quantum setting, where comparable martingale convergence results remain absent. To address this disparity, we introduce an innovative approach to regret bound analysis for quantum RL. Remarkably, our methodology circumvents the reliance on martingale concentration bounds, a staple in classical RL analysis, thus navigating uncharted territories in the quantum realm.

Moreover, it is important to highlight that mean estimation in quantum mechanics results in state collapse, which makes it challenging to perform estimation of state transition probabilities  over multiple epochs. In this study, we present a novel approach to estimating state transition probabilities that explicitly considers the quantum state collapsing upon measurement.

To this end, the following are the major contributions of this work:
\begin{enumerate}
    \item We introduce  Quantum-UCRL (Q-UCRL), a model-based, infinite horizon, optimism-driven quantum Reinforcement Learning (QRL) algorithm. Drawing inspiration from classical predecessors like the UCRL and UC-CURL algorithms \cite{auer2008near,agarwal2021concave,agarwal2023reinforcement}, Q-UCRL is designed to integrate an agent's optimistic policy acquisition and an adept quantum mean estimator, and is the first quantum algorithm for infinite horizon average reward RL with provable guarantees. 
    
    \item Employing meticulous theoretical analysis, we show that the Q-UCRL algorithm achieves an exponential improvement in the regret analysis. Specifically, it attains a regret bound of $\tilde{\mathcal{O}}(1)$, 
    breaking through the theoretical classical lower bound of $\Omega(\sqrt{T})$ across $T$ online rounds. The analysis is based on the novel quantum Bellman error based analysis (introduced in classical RL in \cite{agarwal2021concave} for optimism based algorithm), where the  difference between the performance of a policy on two different MDPs is bounded by long-term
    averaged Bellman error and the quantum mean estimation is used for improved Bellman error. Further, the analysis avoids dependence on martingale concentration bounds and incorporates the phenomenon of state collapse following measurements, a crucial aspect in quantum computing analysis.
   
    \item We design a momentum-based estimator in equations \eqref{eq: p tilde formula}-\eqref{n value} that fuses each epoch’s fresh quantum mean estimate with all past estimates using counts-based weights. The resulting algorithm preserves information that would otherwise be lost to quantum collapse, enabling accuracy to accumulate epoch-by-epoch and marking the first reuse-of-information technique fully compatible with quantum-measurement constraints. This method is essential for integrating quantum speedups into model-based RL frameworks and represents a core technical novelty of our work.

\end{enumerate}  
To the best of our knowledge, these are the first results for quantum speedups for infinite horizon MDPs with average reward objective. 

\section{Related Work and Preliminaries}\label{sec: related_work}
\textbf{Infinite Horizon Reinforcement Learning:} Regret analysis of infinite horizon RL in the classical setting has been widely studied with average reward objective in both model-based settings and model-free settings \cite{wei2020model, bai2024regret, hong2025reinforcement, ganesh2025order, ganesh2025sharper}. In model-based methods, a prominent principle that underpins algorithms tailored for this scenario is the concept of ``optimism in the face of uncertainty" (OFU). In this approach, the RL agent nurtures optimistic estimations of value functions and, during online iterations, selects policies aligned with the highest value estimates \cite{fruit2018efficient,auer2008near,agarwal2021concave}. Additionally, it's noteworthy to acknowledge that several methodologies are rooted in the realm of posterior sampling, where the RL agent samples an MDP from a Bayesian Distribution and subsequently enacts the optimal policy \cite{osband2013more,agrawal2017optimistic,agarwal2022multi,agarwal2023reinforcement}. In our study, we follow the model-based approach and embrace the OFU-based algorithmic framework introduced in \cite{agarwal2021concave}, and we extend its scope to an augmented landscape where the RL agent gains access to supplementary quantum information. Furthermore, we render a mathematical characterization of regret, revealing a $\Tilde{\mathcal{O}}(1)$ bound, which in turn underscores the merits of astutely processing the quantum signals within our framework.

\textbf{Quantum Mean Estimation:} The realm of mean estimation revolves around the identification of the average value of samples stemming from an unspecified distribution. Of paramount importance is the revelation that quantum mean estimators yield a quadratic enhancement when juxtaposed against their classical counterparts \cite{montanaro2015quantum,hamoudi2021quantum}. The key reason for this improvement is based on the quantum amplitude amplification, which allows for suppressing certain quantum states w.r.t. the states that are desired to be extracted 
 \cite{brassard2002quantum}. In Appendix B, we present a discussion around Quantum Amplitude Estimation and its applications in QML. 


In the following, we introduce the definition of key elements and results pertaining to quantum mean estimation that are critical to our setup and analysis. First, we present the definition of a classical random variable and the quantum sampling oracle for performing quantum experiments.

\begin{definition}[Random Variable, Definition 2.2 of \citep{cornelissen2022near}]
	A finite random variable can be represented as $X: \Omega \to E$ for some probability space $(\Omega, \mathbb{P})$, where $\Omega$ is a finite sample set, $\mathbb{P}:\Omega\to[0, 1]$ is a probability mass function and $E\subset \mathbb{R}$ is the support of $X$.  $(\Omega, \mathbb{P})$ is frequently omitted when referring to the random variable $X$.
\end{definition} 

{To perform quantum mean estimation, we provide the definition of quantum experiment. This is analogous to the classical random experiments
\begin{definition}[Quantum Experiment] Consider a random variable $X$ on a probability space $(\Omega, 2^\Omega, \mathbb{P})$. Let $\mathcal{H}_\Omega$ be a Hilbert space with basis states $\{\lvert\omega\rangle\}_{\omega\in\Omega}$ and fix a unitary $\mathcal{U}_{\mathbb{P}}$ acting on $\mathcal{H}_\Omega$ such that
\[
\mathcal{U}_{\mathbb{P}} : \lvert 0\rangle \mapsto \sum_{\omega\in\Omega} \sqrt{\mathbb{P}(\omega)}\lvert\omega\rangle
\]
assuming $0 \in \Omega$. We define a \emph{quantum experiment} as the process of applying the unitary $\mathcal{U}_{\mathbb{P}}$ or its inverse $\mathcal{U}_{\mathbb{P}}^{-1}$ on any state in $\mathcal{H}_\Omega$.
\end{definition}}


 

{The unitiary $\mathcal{U}_{\mathbb{P}}$ provides the ability to query samples of the random variable in superpositioned states, which is the essence for speedups in quantum algorithms. To perform quantum mean estimation for values of random variables \cite{cornelissen2022near,hamoudi2021quantum,montanaro2015quantum}, an additional quantum evaluation oracle would be needed.} 
{
 \begin{definition}[Quantum Evaluation Oracle] \label{def: bin oracle}
	Consider a finite random variable $X : \Omega \to E$ on a probability space $(\Omega, 2^\Omega, \mathbb{P})$. Let $\mathcal{H}_\Omega$ and $\mathcal{H}_E$ be two Hilbert spaces with basis states $\{|\omega\rangle\}_{\omega \in \Omega}$ and $\{|x\rangle\}_{x \in E}$ respectively. We say that a unitary $\mathcal{U}_X$ acting on $\mathcal{H}_\Omega \otimes \mathcal{H}_E$ is a quantum evaluation oracle for $X$ if
\[
\mathcal{U}_X : |\omega\rangle|0\rangle \mapsto |\omega\rangle|X(\omega)\rangle
\]
for all $\omega \in \Omega$, assuming $0 \in E$.
\end{definition}}

Unlike a classical sample, which discloses only one successor state, the quantum oracle stores the entire
 distribution in coherent superposition, offering far greater information richness per query. We note that this above form of quantum oracle, where the agent performs a quantum experiment to query the oracle and obtain a superpositioned quantum sample, is widely adopted in theoretical quantum learning literature, including optimizations \citep{sidford2023quantum}, searching algorithms \cite{grover1996fast}, and episodic RL \cite{zhongprovably}. These quantum experiment frameworks are key mechanisms that enable quantum speedups in sample complexity and regret analysis. We now present a key quantum mean estimation result that will be carefully employed in our algorithmic framework to utilize the quantum superpositioned states collected by the RL agent. One of the crucial aspects of Lemma \ref{lem: SubGau} is that quantum mean estimation converges at the rate $\mathcal{O}(\frac{1}{n})$ as opposed to classical benchmark convergence rate $\mathcal{O}(\frac{1}{\sqrt{n}})$ for $n$ number of samples, therefore estimation efficiency quadratically.   

\begin{lemma}[Quantum multivariate bounded estimator, Theorem 3.3 of \citep{cornelissen2022near}]
	\label{lem: SubGau}
	{Let $X$ be a d-dimensional bounded random variable such that $||X||_2 \leq 1$. Given three reals $L_2 \in (0,1]$, $\delta \in (0,1)$ and $n \geq 1$ such that $\mathbb{E}[||X||_2] \leq L_2$, the multivariate bounded estimator $\texttt{QBounded}_d(X,L_2,n,\delta)$ obtained by Algorithm \ref{algo: QBounded} performs $f(n,d) = \mathcal{O}\Big(n \log^{1/2}(n\sqrt{d}) \Big)$ quantum experiments and outputs a mean estimate $\hat{\mu}$ of $\mu = \mathbb{E}[X]$ such that,}
	\begin{align}
		{\mathbb{P}\left[||\hat{\mu}-\mu||_{\infty}\le \frac{\sqrt{L_2} \log(d/\delta)}{n}\right]\ge 1-\delta.} \label{eq: Subgaussian}
	\end{align}
\end{lemma}

\begin{algorithm}[ht]
\caption{{$\texttt{QBounded}_d(X,L_2,n,\delta)$ (Algorithm 1 in \cite{cornelissen2022near})}}
\label{algo: QBounded}
\begin{algorithmic}[1]

\IF{\(n \leq \frac{\log(d/\delta)}{\sqrt{L_2}}\)} 
\STATE Output \(\hat{\mu} = 0\).
\ENDIF
\STATE Set \(\alpha = \frac{1}{\sqrt{\log(400\pi n\sqrt{d})}}\) and \(m = 2^{\left\lceil \log\left(\frac{8\pi n}{\alpha \sqrt{L_2 \log(d/\delta)}}\right) \right\rceil}\).
\STATE Set $G = \left\{ \frac{j}{m} - \frac{1}{2} + \frac{1}{2m} : j \in \{0, \dots, m-1\} \right\}^d \subseteq \left(-\frac{1}{2}, \frac{1}{2}\right)^d$

\FOR{\(k = 1, \dots, \lceil 18 \log(d/\delta) \rceil\)}
    \STATE Compute the uniform superposition \(|G\rangle := \frac{1}{\sqrt{m d/2}} \sum_{u \in G}|u\rangle\) over \(G\).
    \STATE Compute the state \(|\psi\rangle := \widetilde{P}_{X, L_2, m, \alpha, \epsilon}|G\rangle\) in \(H_G \otimes H_{aux}\), where \(\widetilde{P}_{X, L_2, m, \alpha, \epsilon}\) is the directional mean oracle constructed by the quantum evaluation oracle in Proposition 3.2 \cite{cornelissen2022near} with \(\epsilon = 1/25\).
    \STATE Compute the state \(|\phi\rangle := (QFT_G^{-1} \otimes I_{aux})|\psi\rangle\) where the unitary \(QFT_G : |u\rangle \mapsto \frac{1}{\sqrt{m d/2}} \sum_{v \in G} e^{2\pi i \langle u, v \rangle}|v\rangle\) is the quantum Fourier transform over \(G\).
    \STATE Measure the \(H_G\) register of \(|\phi\rangle\) in the computational basis and let \(v^{(k)} \in G\) denote the obtained result. Set \(\mu^{(k)} = \frac{2\pi v^{(k)}}{\alpha}\).
\ENDFOR
\STATE Output the coordinate-wise median \(\hat{\mu} = \text{median}(\mu^{(1)}, \dots, \mu^{(\lceil 18 \log(d/\delta) \rceil)})\).
\end{algorithmic}
\end{algorithm}

\textbf{Quantum Reinforcement Learning:} Within the realm of QRL, a prominent strand of previous research showcases exponential enhancements in regret through amplitude amplification-based methodologies applied to the Quantum Multi-Armed Bandits (Q-MAB) problem \cite{wang2021quantum,casale2020quantum, wan2023quantum, wu2023quantum}. Nevertheless, the theoretical underpinnings of the aforementioned Q-MAB approaches do not seamlessly extend to the QRL context since there is no state evolution in bandits.

A recent surge of research interest has been directed towards Quantum Reinforcement Learning (QRL) \cite{jerbi2021quantum,dunjko2017advances,paparo2014quantum}. It is noteworthy to emphasize that the aforementioned studies do not provide an exhaustive mathematical characterization of regret. \cite{zhongprovably, ganguly2023quantum} demonstrated that QRL can provide logrithmic regret in the episodic MDP settings. Ours is the first study of infinite horizon MDP with average reward objective. 


\section{Problem Formulation}\label{sec: prob_formulation}

\noindent We consider the problem of Reinforcement Learning (RL) in an infinite horizon Markov Decision Process characterized by $\mathcal{M} \triangleq (\mathcal{S}, \mathcal{A}, P, r, D)$, wherein $\mathcal{S}$ and $\mathcal{A}$ represent finite collections of states and actions respectively {with $|\mathcal{S}| = S$ and $|\mathcal{A}| = A$}, pertaining to RL agent's interaction with the unknown MDP environment. $P(s' | s,a) \in [0,1]$ denotes the transition probability for next state $s' \in \mathcal{S}$ for a given pair of previous state and RL agent's action, i.e., $(s,a) \in \mathcal{S} \times \mathcal{A}$. Further, $r : \mathcal{S} \times \mathcal{A} \rightarrow [0,1]$ represents the reward collected by the RL agent for state-action pair $(s,a)$. $D$ is the diameter of the MDP $\mathcal{M}$. In the following, we first present the additional quantum transition oracle that the agent could access during its interaction with the unknown MDP environment at every RL round. Here, we would like to emphasize that unknown MDP environment implies that matrix $P$ which encapsulates the transition dynamics of the underlying RL environment is not known beforehand. \\


\noindent \textbf{Quantum Computing Basics:} In this section, we give a brief introduction to the concepts that are most related to our work based on \cite{nielsen2010quantum} and \cite{wu2023quantum}. Consider an $m$ dimensional Hilbert space $\mathbb{C}^m$, a quantum state $|x\rangle = (x_1, ..., x_m)^T$ can be seen as a vector inside the Hilbert space with $\sum_i |x_i|^2 = 1$. Furthermore, for a finite set with $m$ elements $\xi = \{\xi_1, ..., \xi_m\}$, we can assign each $\xi_i \in \xi$ with a member of an orthonormal basis of the Hilbert space $\mathbb{C}^m$ by

\begin{equation}
    \xi_i \mapsto |\xi_i\rangle \equiv e_i
\end{equation}

\noindent where $e_i$ is the $i$th unit vector for $\mathbb{C}^m$. Using the above notation, we can express any arbitrary quantum state $|x\rangle = (x_1, ..., x_m)^T$ by elements in $\xi$ as:
\begin{equation}
    |x\rangle = \sum_{n=1}^m x_n |\xi_n \rangle
\end{equation}
\noindent where $|x\rangle$ is the quantum superposition of the basis $|\xi_1\rangle, ..., |\xi_m\rangle$ and we denote $x_n$ as the amplitude of $|\xi_n\rangle$. To obtain a classical information from the quantum state, we perform a measurement and the quantum state would collapse to any basis  $|\xi_i\rangle$ with probability $|x_i|^2$. In quantum computing, the quantum states are represented by input or output registers made of qubits that could be in superpositions. \\

\noindent \textbf{Quantum transition oracle:} We utilize the concepts and notations of quantum computing in \cite{wang2021quantum,wiedemann2022quantum,ganguly2023quantum,jerbi2022quantum} to construct the quantum sampling oracle of RL environments and capture agent's interaction with the unknown MDP environment.
We now formulate the equivalent quantum-accessible RL environments for our classical MDP $\mathcal{M}$. For an agent at step $t$ in state $s_t$ and with action $a_t$, we construct the quantum sampling oracles {for variables of the next state $P(\cdot|s_t, a_t)$}. Specifically, suppose there are two Hilbert spaces $\mathcal{\bar{S}} = \mathbb{C}^{|\mathcal{S}|}$ and $\mathcal{\bar{A}} = \mathbb{C}^{|\mathcal{A}|}$ containing the superpositions of the classical states and actions. We represent the computational bases for $\mathcal{\bar{S}}$ and $\mathcal{\bar{A}}$ as $\{|s\rangle\}_{s \in \mathcal{S}}$ and $\{|a\rangle\}_{a \in \mathcal{A}}$. We assume that we can implement the quantum sampling oracle {$\mathcal{U}_P$ of the MDP's transitions as follows:}

{The quantum evaluation oracle for the transition probability (quantum transition oracle) $\mathcal{U}_P$ which at step $t$, returns the superposition over $s' \in \mathcal{S}$ according to $P(s' | s_t, a_t)$, the probability distribution of the next state given the current state $|s_t\rangle$ and action $|a_t\rangle$ is defined as:}
    $$
    \mathcal{U}_P : |s_t\rangle \otimes|a_t\rangle \otimes |0\rangle \rightarrow |\bf{\psi}\rangle^t
    $$
    {Where $|\bf{\psi}\rangle^t = |s_t\rangle \otimes |a_t\rangle \otimes \sum_{s' \in \mathcal{S}} \sqrt{P(s'|s_t,a_t)} |s'\rangle$.}

\begin{figure}[H]
    \centering
    \includegraphics[width=0.45\textwidth]{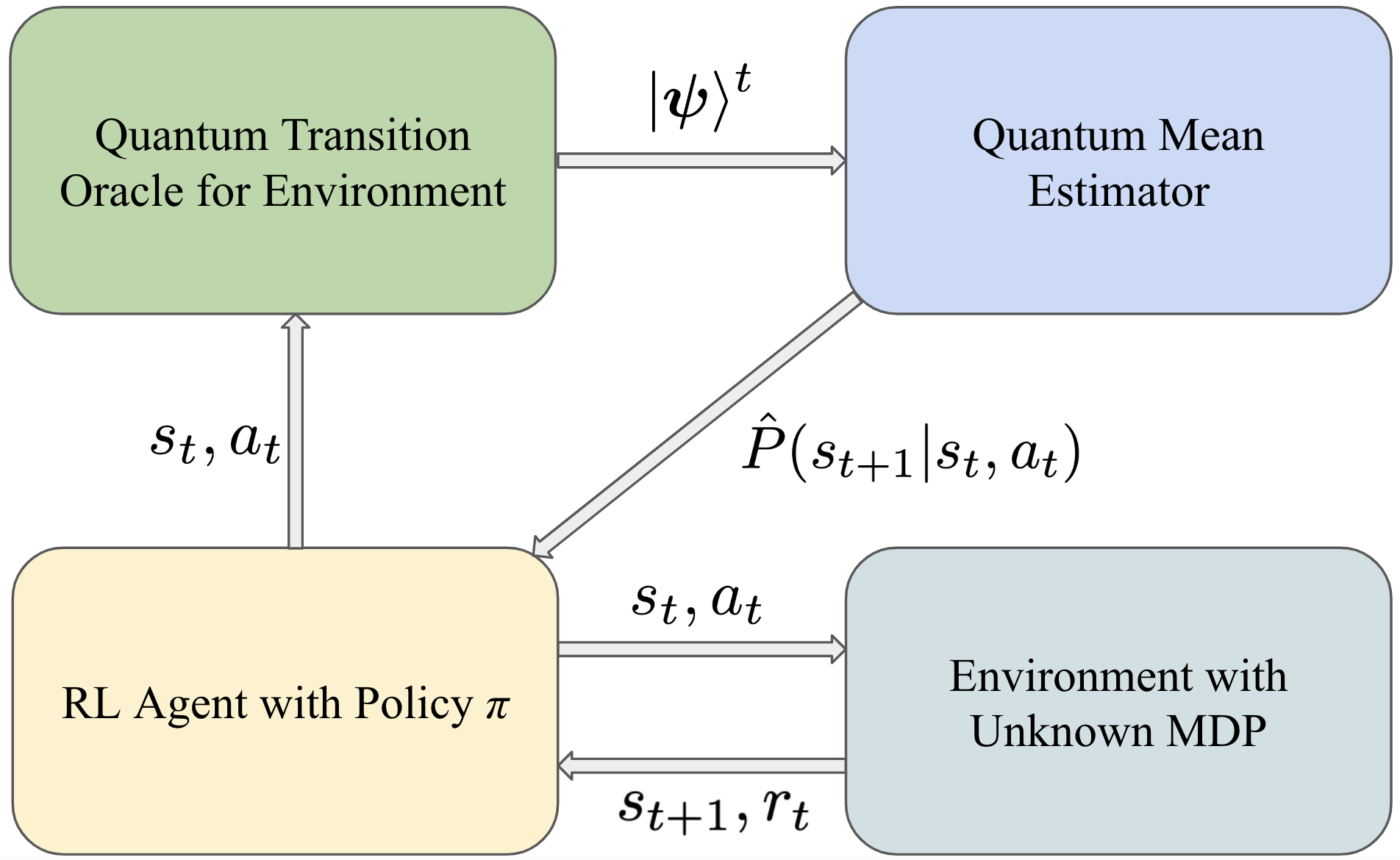}
    \caption{ Agent's interaction at round $t$ with the MDP Environment and accessible quantum transition oracle}
    \label{new_model}
    \vspace{-.1in}
\end{figure}

\noindent As described in Fig \ref{new_model}, at step $t$, the agent plays an action $a_{t}$ based on current state $s_t$ according to its policy $\pi : \mathcal{S} \rightarrow \mathcal{A}$. Consequently, the unknown MDP Q-environment shares the next state information and reward i.e., $\{s_{t+1}, r_t\}$. Additionally, the quantum transition oracle first encodes $s_t$ and $a_t$ into the corresponding basis of the quantum Hilbert spaces, producing basis $|s_t\rangle$ and $|a_t\rangle$, {then it encodes the superpositioned quantum state $|\bf{\psi}\rangle^{t}$ from the quantum transition oracle $\mathcal{U}_P$. $|\bf{\psi}\rangle^{t}$ would be used for the Quantum Mean Estimator, \texttt{QBounded}, which in turn improves transition probability estimation and leads to exponentially better regret. We note that one copy of $|\bf{\psi}\rangle^{t}$ is obtained at each $(s_t,a_t)$. Given that a measurement collapses the state, we can only perform one measurement on it. So, the quantum mean estimator will only perform measurements at epoch ends, as will be described in the algorithm. }\\

\noindent \textbf{RL regret formulation:} Consider the RL agent following some policy $\pi : \mathcal{S} \rightarrow \mathcal{A}$ on the unknown MDP Q-environment  $\mathcal{M}$. Then, as a result of playing policy $\pi$, the agent's long term average reward is defined next according to \cite{agarwal2021concave}.
\begin{align}
    \Gamma_{\pi}^{P} = \lim_{T \rightarrow \infty} \mathbb{E}_{\pi, P} \big[\frac{1}{T}\sum_{t=0}^{T-1} ~r(s_t, a_t) \big], \label{eq: long term avg reward def}
\end{align}
where $\mathbb{E}_{\pi, P}[\cdot]$ is the expectation taken over the trajectory $\{(s_t, a_t)\}_{t \in [0, T-1]}$ by playing policy $\pi$ on $\mathcal{M}$ at every time step. In this context, we state the definitions of value function  for MDP $\mathcal{M}$ next:
\begin{align}
    V_{\pi}^{{P}}(s;\gamma) &= \mathbb{E}_{a \sim \pi} \Big[Q_{\pi}^{{P}}(s,a;\gamma)  \Big], \label{eq: val func 0}\\
    & \hspace{-5mm} = \mathbb{E}_{a \sim \pi} \Big[r(s,a) + \gamma \sum_{s' \in \mathcal{S}} P(s'| s,a) V_{\pi}^{{P}}(s';\gamma) \Big] \label{eq: val func -1}
\end{align}
Then, according to \cite{puterman2014markov}, $\Gamma_{\pi}^{P}$ can be represented as the following:
\begin{align}
    \Gamma_{\pi}^{P} &= \lim_{\gamma \rightarrow 1} (1-\gamma) V_{\pi}^{{P}}(s;\gamma), ~\forall s \in \mathcal{S}, \label{eq: value puterman 0} \\
    & = \sum_{s \in \mathcal{S}} \sum_{a \in \mathcal{A}} \rho_{\pi}^{P} (s,a) \overline{r}(s,a), \label{eq: value puterman 1} 
\end{align}
wherein, $V_{\pi}^{P} (\cdot;\cdot)$ is the discounted cumulative average reward by following policy $\pi$, $\rho_{\pi}^{P}$ is the steady-state occupancy measure (i.e., $\rho_{\pi}^{P}(s,a) \in [0,1], ~ \sum_{s \in \mathcal{S}} \sum_{a \in \mathcal{A}} \rho_{\pi}^{P}(s,a) = 1$), $\bar{r}(s,a)$ is the steady-state long-term average reward of pair $(s,a)$. In the following, we introduce key assumptions crucial from the perspective of our algorithmic framework and its analysis. In this context, we introduce the notations $\{P_{\pi,s}^{t}, T_{\pi, ~s \rightarrow s'} \}$ to denote the $t$ time step probability distribution obtained by playing policy $\pi$ on the unknown MDP $\mathcal{M}$ starting from some arbitrary state $s \in \mathcal{S}$; and, the actual time steps to reach $s'$ from $s$ by playing policy $\pi$ respectively. Next, we state an assumption pertaining to ergodicity of MDP $\mathcal{M}$.  
\begin{assumption}[Finite MDP mixing time] \label{assumption: mixing time}
We assume that unknown MDP environment $\mathcal{M}$ has finite mixing time, which mathematically implies:
    \begin{align}
        T_{\texttt{mix}} \triangleq \max_{\pi, (s,s') \in \mathcal{S} \times \mathcal{S}} ~\mathbb{E} [T_{\pi, ~s \rightarrow s'}] < \infty, 
    \end{align}
    where $T_{\pi, ~s \rightarrow s'}$ implies the number of time steps to reach a state $s'$ from an initial state $s$ by playing some policy $\pi$.
\end{assumption}
\begin{assumption} \label{assumption: known reward} The reward function $r(\cdot, \cdot)$ is known to the RL agent.  
\end{assumption}
We emphasize that Assumption \ref{assumption: known reward} is a commonly used assumption in RL algorithm development owing to the fact that in most setups rewards are constructed according to the underlying problem and is known beforehand  \cite{azar2017minimax,agarwal2019reinforcement}. Next, we define the cumulative regret accumulated by the agent in expectation across $T$ time steps as follows:
\begin{align}
    \mathcal{R}_{[0, T-1]} \triangleq T.\Gamma_{\pi^{*}}^{P} - \mathbb{E} \Big[ \sum_{t=0}^{T-1} r(s_t, a_t)\Big], \label{eq: regret defn}
\end{align}
where the agent generates the trajectory $\{(s_t, a_t\}_{t \in [0, T-1]}$  by playing policy $\pi$ on $\mathcal{M}$, and $\pi^{*}$ is the optimal policy that maximizes the long term average reward defined in Eq. \eqref{eq: long term avg reward def}. Finally, the RL agent's goal is to determine a policy $\pi$ which minimizes $\mathcal{R}_{[0, T-1]}$.
\section{Algorithmic Framework}\label{sec: algorithm}
In this section, we first delineate the key intuition behind our methodology followed by formal description of Q-UCRL algorithm. If the agent was aware of the true transition probability distribution $P$, then it would simply solve the following optimization problem for the optimal policy.
\begin{align}
    & \max_{\{\rho(s,a)\}_{(s,a) \in \mathcal{S} \times \mathcal{A}} } ~\sum_{(s,a) \in \mathcal{S} \times \mathcal{A}} r(s,a) \rho(s,a) \label{eq: opt_obj}
\end{align}

With the following constraints,
\begin{align}
    & \sum_{s,a} \rho(s,a) = 1, \rho(s,a) \geq 0 \quad \forall (s,a) \in \mathcal{S} \times \mathcal{A}\label{eq: transition constraint 1} \\
    & \sum_{a \in \mathcal{A}} \rho(s',a) = \sum_{s, a} P(s'| s,a) \rho(s,a), \forall s' \in \mathcal{S} \label{eq: transition constraint 2}
\end{align}
Note that the objective in Eq. \eqref{eq: opt_obj} resembles the definition of steady-state average reward formulation in  Eq. \eqref{eq: long term avg reward def}. Furthermore, Eq. \eqref{eq: transition constraint 1}, \eqref{eq: transition constraint 2} preserve the properties of the transition model. Consequently, once it obtains $\{\rho^*(s,a)\}_{(s,a) \in \mathcal{S} \times \mathcal{A}}$ by solving Eq. \eqref{eq: opt_obj}, it can compute the optimal policy as:
\begin{align}
    \pi^{*}(a|s) = \frac{\rho^{*}(s,a)}{\sum_{\tilde{a} \in \mathcal{A}} \rho^{*}(s,\tilde{a})} \quad \forall s, a
\end{align}
However, in the absence of knowledge pertaining to true $P$, we propose to solve an optimistic version of Eq. \eqref{eq: opt_obj} that will update RL agent's policy in epochs. Specifically, at the start of an epoch $e$, the following optimization problem will be solved:


\begin{align}
    & \max_{\{\rho(s,a), P_e (\cdot|s,a)\}_{(s,a) \in \mathcal{S} \times \mathcal{A}} } ~\sum_{(s,a) \in \mathcal{S} \times \mathcal{A}} r(s,a) \rho(s,a), \label{eq: optimistic opt 0}
\end{align}
With the following constraints,
\begin{align}
    & \sum_{s,a} \rho(s,a) = 1, \rho(s,a) \geq 0 \quad \forall (s,a) \in \mathcal{S} \times \mathcal{A}\\
    & \sum_{s' \in \mathcal{S}} P_e(s'|s,a) = 1, P_e(\cdot|s,a) \geq 0 ~\forall (s,a) \in \mathcal{S} \times \mathcal{A} \\
    &\sum_{a \in \mathcal{A}} \rho(s',a)  = \sum_{(s,a) \in \mathcal{S} \times \mathcal{A}} {P}_{e}(s'| s,a) \rho(s,a), \forall s' \in \mathcal{S} \label{eq:mainnon}\\
    &  \|{P}_e(\cdot | s,a) - \hat{P}_e(\cdot|s,a) \|_{1} \leq \frac{7 S \log(S^2 A t_e)}{\max\{1,N_{e}(s,a)\}} \label{eq: optimistic opt -1}
\end{align}
 Observe that Eq. \eqref{eq: optimistic opt 0} increases the search space for policy in a  carefully computed neighborhood of certain transition probability estimates $\hat{P}$ via Eq. \eqref{eq: optimistic opt -1} at the start of epoch $e$. 
 Also in Eq. \eqref{eq: optimistic opt -1}, {recall that} $S = |\mathcal{S}|$ is the number of states, $A = |\mathcal{A}|$ is the number of actions, $N_e(s,a)$ is the number of visitations of state-action pair $(s,a)$ upto the end of epoch $e$. Note that, solving Eq. \eqref{eq: optimistic opt 0} is equivalent to obtaining an optimistic policy as it extends the search space of original problem Eq. \eqref{eq: opt_obj}. Furthermore, the computation of $\hat{P}_e$ is key to our framework design. Since performing quantum mean estimation to quantum samples would require measuring and results in quantum samples collapsing, each quantum sample can only be used once. We emphasize that the quantum mean estimator will only run at the end of the epoch, performing measurements to update the transition probability for the next epoch. 
 
 {\bf Details on Quantum Mean Estimator: } At the end of the epoch, we use all $(s_i, a_i)$ observed in the samples, and perform quantum mean estimation on all the samples $|\bf{\psi}\rangle^{t}$  obtained for that $(s_i, a_i)$.  More precisely, {$\tilde{P}_{e}(\cdot|s_i,a_i)$ is the quantum mean estimation for the random variable $P_{e}(\cdot|s_i,a_i)$ performed by \texttt{QBounded} (Algorithm \ref{algo: QBounded}) using the quantum samples $\{|{\psi}\rangle\}^e_i$ observed in the previous epoch with the total number being  $\nu_{e-1}(s_i,a_i)$:}
 \begin{equation} \label{eq: quantum sample notation} 
    {\{|{\psi}\rangle\}^e_i = \{|{\psi}\rangle^t \,| \; t_{e-1}^{start} \leq t \leq t_e^{start}-1, (s_t, a_t) = (s_i, a_i)\}}
\end{equation} 
where  $t_{e}^{start}$ is the starting step of epoch $e$.  Using these samples, the transition probability is obtained as
\begin{equation} \label{eq: p tilde formula} 
    {\tilde{P}_e(\cdot|s_i,a_i) = \texttt{QBounded}_S(\{P_e(\cdot|s_i,a_i), L_2, n_e^{s_i,a_i}, \delta_e) }
\end{equation}

For a specific state-action pair $(s_i, a_i)$, we construct an estimator $\hat{P}_e(\cdot|s_i,a_i)$ used in epoch $e$ that consists of a weighted average of the estimator $\hat{P}_{e-1}(\cdot|s_i,a_i)$ and the estimator obtained using the quantum samples in the latest completed epoch $\tilde{P}_{e}(\cdot|s_i,a_i)$ as follows:

 \begin{equation} \label{eq: p hat formula} 
  \hat{P}_e(\cdot|s_i,a_i) =
    \begin{cases}
      0 & \text{if $e = 1$} \\
      \tilde{P}_e(\cdot|s_i,a_i) & \text{if $e = 2$}\\
      \hat{w}\hat{P}_{e-1}(\cdot|s_i,a_i) + \tilde{w}\tilde{P}_e(\cdot|s_i,a_i) & \text{if $e \geq 3$}\\
    \end{cases}    
\end{equation}

where

\begin{equation} \label{eq: weighted mean1} 
    \hat{w} = \frac{N_{e-1}(s_i, a_i)}{N_{e-1}(s_i, a_i)+\nu_{e-1}(s_i, a_i)}
\end{equation}

\begin{equation} \label{eq: weighted mean2} 
    \tilde{w} = \frac{\nu_{e-1}(s_i, a_i)}{N_{e-1}(s_i, a_i)+\nu_{e-1}(s_i, a_i)}
\end{equation}

{In particular, given that $\nu_{e-1}(s_i,a_i)$ is the maximum number of quantum experiments we can perform in Eq. \eqref{eq: p tilde formula}, we set 
\begin{equation} \label{n value} 
    n_e^{s_i,a_i} = \left\lfloor \frac{\nu_{e-1}(s_i,a_i)}{c\log ^ {1/2}(T \sqrt{S})}\right\rfloor
\end{equation}
for $c \in \mathbb{R}$ satisfying $cn\log^{1/2}(n\sqrt{S}) \geq f(n,S)  \;\forall n$ where $f(n, S)$ is defined in Lemma \ref{lem: SubGau}.  Such choice of $c$ would allow $\nu_e(s,a)$ samples to be used  for analysis of our proposed algorithm. Note that the choice of $c$ is achievable because $cn\log^{1/2}(n\sqrt{S})$ and $f(n,S)$ are  the same order. Since $\mathbb{E}[||P_e(\cdot|s_i,a_i)||_2] \leq 1$, we set $L_2 = 1.$} Different from the works in classical approaches based on \cite{jaksch2010near}, our algorithm utilizes each quantum sample only once for estimating all $\hat{P}_e$ throughout the entire process. This strategy fits with our quantum framework since quantum samples would collapse once being measured.
 In Appendix A, we show that Eq. \eqref{eq: optimistic opt 0} is a convex optimization problem, therefore standard optimization solvers can be employed to solve it. Finally, the agent's policy at the start of epoch $e$ is given by:  
\begin{equation}
    \pi_{e}(a|s) = \frac{\rho_{e}(s,a)}{\sum_{\tilde{a} \in \mathcal{A}} \rho_{e}(s,\tilde{a})} \label{eq: epoch policy}
\end{equation}
Next, we formally present Q-UCRL in Algorithm \ref{algo: QUCRL}. 

\if 0\begin{algorithm}[ht]
\caption{Q-UCRL}
\label{algo: QUCRL}
\begin{algorithmic}[1]
\State \textbf{Inputs:} $\mathcal{S}$, $\mathcal{A}$, $r(\cdot, \cdot)$.
	\State Set $t \leftarrow 1, ~e \leftarrow 1, ~\delta_e \leftarrow 0, t_e \leftarrow 0, t_e^{start} \leftarrow 1$
\end{algorithmic}
\end{algorithm} 
\fi 
\begin{algorithm}[ht]
\caption{Q-UCRL}
\label{algo: QUCRL}
\begin{algorithmic}[1]
	\STATE \textbf{Inputs:} $\mathcal{S}$, $\mathcal{A}$, $r(\cdot, \cdot)$.
	\STATE Set $t \leftarrow 1, ~e \leftarrow 1, ~\delta_e \leftarrow 0, t_e \leftarrow 0, t_e^{start} \leftarrow 1$
 \STATE Set $\mu(s, a, s') \leftarrow 0 ~\forall (s,a,s') \in \mathcal{S} \times \mathcal{A} \times \mathcal{S}$.
 \STATE Set $\hat{P}_e(\cdot|s,a) \leftarrow 0 ~\forall (s,a) \in \mathcal{S} \times \mathcal{A}$.
     \FOR{$(s,a) \in \mathcal{S} \times \mathcal{A}$}
     \STATE Set $\nu_{e}(s,a) \leftarrow 0, ~N_{e}(s,a) \leftarrow 0$.
     \ENDFOR
    \STATE Obtain $\pi_{e}$ by solving Eq. \eqref{eq: optimistic opt 0} and Eq. \eqref{eq: epoch policy}.
	\FOR {$t=1,2, \dots $}
     \STATE Observe state $s_t$ and play action  $a_{t} \sim \pi_{e}(\cdot | s_t)$.
     \STATE Observe $s_{t+1}, r(s_t, a_t)$ and quantum samples $|\bf{\psi}\rangle^{t}$.
     \STATE Update $\nu_{e}(s_t, a_t) \leftarrow \nu_{e}(s_t, a_t) + 1$.
     \STATE Update $\mu(s_t, a_t, s_{t+1}) \leftarrow \nu_{e}(s_t, a_t, s_{t+1}) + 1$.
     \STATE Set $t_e \leftarrow t_e + 1$.
     \IF{$\nu_{e}(s_t, a_t) = \max \{1, N_{e}(s_t,a_t) \}$}
     \FOR{$(s,a) \in \mathcal{S} \times \mathcal{A}$}
     \STATE $N_{e+1}(s,a) \leftarrow N_{e}(s,a) + \nu_{e}(s,a)$.
     \ENDFOR
     \STATE Set $e \leftarrow e + 1, ~\nu_{e}(s,a) \leftarrow 0, ~\delta_e \leftarrow \frac{1}{S^2 A t_e^7}$.
     \STATE Set $t_e \leftarrow 0, \; t_e^{start} \leftarrow t+1$.
     \STATE Obtain $\hat{P}_e(\cdot|s,a)$ by Eq. \eqref{eq: p hat formula} $\forall(s,a) \in \mathcal{S} \times \mathcal{A}$
     \STATE Obtain $\pi_{e}$ by solving Eq. \eqref{eq: optimistic opt 0} and Eq. \eqref{eq: epoch policy}.
     \ENDIF
     \ENDFOR
\end{algorithmic}
\end{algorithm} 


Algorithm \ref{algo: QUCRL} proceeds in epochs, and each epoch $e$ contains multiple rounds of RL agent's interaction with the environment. At each $t$, by playing action $a_t$ for insantaneous state $s_t$ through current epoch policy $\pi_{e}$, the agent collects next state configuration consisting of classical signals $s_{t+1}, r(s_t, a_t)$ and a quantum sample $|\bf{\psi}\rangle^{(t)}$. Additionally, the agent keeps track of its visitations to the different states via the variables $\{\nu_{s, a}, \mu_{s, a, s'}\}$. Note that a new epoch is triggered at the ``if" block of Line 15 in algorithm \ref{algo: QUCRL}, which in turn is aligned to the widely used concept of \textit{doubling trick} used in model-based RL algorithm design. Finally, the empirical visitation counts and probabilities up to epoch $e$, i.e., $\{N_{e+1}(s,a), \hat{P}_{e+1}(s'|s,a)\}$ attributes are updated (line 15-21), and the policy is updated via re-solving from Eq. \eqref{eq: optimistic opt 0} - \eqref{eq: optimistic opt -1} and using Eq. \eqref{eq: epoch policy}.
 
\section{Theoretical Results}\label{sec: theory}
In this section, we characterize the cumulative regret for $T$ rounds by running Q-UCRL in a unknown MDP hybrid Q-Environment. In the following, we first present a crucial result bounding the error accumulated till $e$ epochs between the true transition probabilities $\{P(\cdot|s,a) \}$ and the agent's estimates $\{\hat{P}_e(\cdot|s,a) \}$.

\begin{lemma} \label{lemma: P tilde bound} Execution of Q-UCRL (Algorithm \ref{algo: QUCRL}) up to the beginning of epoch $e+1$ with total $e$ completed epochs comprising $t = 1, 2, \ldots, t_e$ rounds,  guarantees that the following holds for $\Tilde{P}_{e+1}$:
\begin{align} \label{eq: P tilde bound}
    \mathbb{P}\bigg[|\tilde{P}_{e+1}(s'|s, a) - P(s'|s,a)| \leq {\frac{C \log(S/\delta)}{\nu_e(s,a)}}\bigg] \geq 1 - \delta
\end{align}
$\forall (s,a,s') \in \mathcal{S} \times \mathcal{A} \times \mathcal{S}$, where $C = {c\log^{1/2}(T\sqrt{S})}$ for some $c \in \mathbb{R}$ {defined in Eq. \eqref{n value}} and $\{ \nu_{e}(s,a) \}$ are the state-action pair visitation counts up in epoch $e$, as depicted in Line 12 of Algorithm \ref{algo: QUCRL}, and $\{P(s'|s,a), \tilde{P}_{e+1}(s'|s,a) \}$ are the actual and estimated transition probabilities by Eq. \eqref{eq: p tilde formula}.

\begin{proof}
    {For all $(s,a) \in \mathcal{S} \times \mathcal{A}$ , since $\tilde{P}_{e+1}$ is obtained by Eq. \eqref{eq: p tilde formula}, we obtain using Lemma \ref{lem: SubGau}:}

    \begin{align} \label{eq: original form}
        \mathbb{P}\bigg[{||\tilde{P}_{e+1}(\cdot|s, a) - P(\cdot|s,a)||_{\infty} \leq \frac{ \log(S/\delta)}{n_{e+1}^{s_i,a_i}}}\bigg] \geq 1 - \delta,
    \end{align}
    {By switching $n_e^{s_i,a_i}$ with $\nu_e{(s,a)}/C$, Eq. \eqref{eq: original form} gives:}
    \begin{align} \label{eq: inf norm bound}
        {\mathbb{P}\bigg[||\tilde{P}_{e+1}(\cdot|s, a) - P(\cdot|s,a)||_{\infty} \leq \frac{ C\log(S/\delta)}{\nu_e(s,a)}\bigg] \geq 1 - \delta}
    \end{align}
    {Eq. \eqref{eq: inf norm bound} can be transformed into Eq. \eqref{eq: P tilde bound} by taking the entry-wise expression.}
    
\end{proof}
\end{lemma}
 Note that the choice of $C$ in Eq. \eqref{eq: P tilde bound} is to ensure that the samples $\nu_e(s,a)$ are sufficient for performing Algorithm \ref{algo: QBounded}.
\begin{lemma} \label{lemma: P hat bound} Execution of Q-UCRL (Algorithm \ref{algo: QUCRL}) up to the beginning of epoch $e+1$ with total $e$ completed epochs comprising $t = 1, 2, \ldots, t_e$ rounds,  guarantees that the following holds for $\Hat{P}_{e+1}$:
\begin{align} \label{eq: P hat bound}
    \mathbb{P}\bigg[|\hat{P}_{e+1}(s'|s, a) - P(s'|s,a)| \leq \frac{eC \log(e{S}/\delta)}{N_{e+1}(s,a)}\bigg] \geq 1 - \delta
\end{align}
$\forall (s,a,s') \in \mathcal{S} \times \mathcal{A} \times \mathcal{S}$, where $C ={c\log^{1/2}(T\sqrt{S})}$ for some $c \in \mathbb{R}$ and $\{ \nu_{e}(s,a) \}$ are the state-action pair visitation counts up in epoch $e$, as depicted in Line 12 of Algorithm \ref{algo: QUCRL}, and $\{P(s'|s,a), \hat{P}_{e+1}(s'|s,a) \}$ are the actual and weighted average of transition probabilities by Eq. \eqref{eq: p hat formula}.

\begin{proof}
    Please refer to Appendix C.
\end{proof}
\end{lemma}

\noindent Lemma \ref{lemma: P tilde bound} and Lemma \ref{lemma: P hat bound} provide the acceleration for mean estimation, which further improve global convergence rate.

\begin{lemma} \label{lemma: model_err_SA} Execution of Q-UCRL (Algorithm \ref{algo: QUCRL}) up to total $e$ epochs comprising $t = 1, 2, \ldots, t_e$ rounds,  guarantees that the following holds for confidence parameter $\delta_e = \frac{1}{S A t_e^7}$ with probability at least $1-1/t_{e}^6$:
\begin{align}
    &\|{P}_{e}(\cdot | s,a ) - {P}(\cdot | s,a ) \|_1, \nonumber\\
    &\leq  \frac{7 S (1+Ce)\log(S^2 A T)+ SCelog{(eS)}}{\max\{1,N_{e}(s,a)\}}, \label{eq: model_est_err}
\end{align}
$\forall (s,a) \in \mathcal{S} \times \mathcal{A}$, where $\{ N_{e-1}^{(s,a)} \}$ is the cumulative state-action pair visitation counts up to epoch $e$, as depicted in Line 17 of Algorithm \ref{algo: QUCRL}, and $\{P(\cdot|s,a), P_e(\cdot|s,a) \}$ are the actual and estimated transition probabilities.
\end{lemma}
\begin{proof}
Please refer to Appendix D.
\end{proof}
\begin{remark} \label{remark: bad event}
    Note that for ${P}_{e}(\cdot | s,a ),{P}(\cdot | s,a )$ we have the bound: $\|{P}_{e}(\cdot | s,a ) - {P}(\cdot | s,a ) \|_1 \leq 2 S$, in the worst case scenario. However, the probability that a \textit{bad event} (i.e., complementary of Eq. \eqref{eq: model_est_err}) happens is at most $\frac{1}{t_e^6}$.
\end{remark}
\begin{lemma}
    [Regret Decomposition] Regret of Q-UCRL defined in Eq. \eqref{eq: regret defn} can be decomposed as follows:
    \begin{align}
        \mathcal{R}_{[0, T-1]} \leq \sum_{e = 1}^{E} T_{e} \Big(\Gamma_{\pi_{e}}^{\tilde{P}_{e}} - \Gamma_{\pi_{e}}^{{P}} \Big). \label{eq: regr decomp final}
    \end{align}
\end{lemma}
\begin{proof}
    Please refer to Appendix E.
\end{proof}
{\bf Regret in terms of Bellman Errors: } Here, we introduce the notion of Bellman errors $B_{\pi}^{{P}_{e}} (s,a)$ at epoch $e$ as follows:
\begin{align}
    B_{\pi}^{{P}_{e}} (s,a) \triangleq  \lim_{\gamma \rightarrow 1} &~[Q_{\pi}^{{P}_e}(s,a;\gamma) - r(s,a) \nonumber \\
    & - \gamma \sum_{s' \in \mathcal{S}}P(s'|s,a)V_{\pi}^{{P}_e}(s;\gamma)], \label{eq: bellman err}
\end{align}
which essentially captures the gap in the average rewards by playing action $a$ in state $s$ accordingly to some policy $\pi$ in one step of the optimistic transition model $P_e$ w.r.t. the actual probabilities $P$. This leads us to the next result that ties the regret expression in Eq. \eqref{eq: regr decomp final}  and the Bellman errors.

 \begin{lemma}  [Lemma 5.2, \cite{agarwal2021concave}] \label{lemma: regr bellman}
The difference in long-term expected reward by playing optimistic policy of epoch $e$, i.e., $\pi_e$ on the optimistic transition model estimates, i.e., $\Gamma_{\pi_{e}}^{{P}_{e}}$, and the expected reward collected by playing $\pi_e$ on the actual model $P$, i.e., $\Gamma_{\pi_{e}}^{{P}}$ is the long-term average of the Bellman errors. Mathematically, we have:  
\begin{align}
    \Gamma_{\pi_{e}}^{\tilde{P}_{e}} - \Gamma_{\pi_{e}}^{{P}} = \sum_{s,a} \rho_{\pi_{e}}^{P}(s,a) B_{\pi_e}^{{P}_{e}} (s,a), \label{eq: bellman err new} 
\end{align}

where $B_{\pi}^{{P}_{e}} (s,a)$ is defined in Eq. \eqref{eq: bellman err}.
\end{lemma}
With Eq. \eqref{eq: bellman err new}, it necessary to mathematically bound  the Bellman errors. To this end, we leverage the following result from \cite{agarwal2021concave} which is stated as Lemma 10 from Appendix F characterizing the Bellman errors in terms of $\tilde{h}(\cdot)$, a quantity that measures the bias contained in the optimistic MDP model.

\begin{align}
    B^{\pi_e, {P}_e}(s,a) \leq \Big\|{P}_e(\cdot|s,a) - {P}(\cdot|s,a) \Big\|_{1} \|\tilde{h}(\cdot) \|_{\infty},  \label{eq: bellman h lemma result 0}
\end{align}
where $\tilde{h}(\cdot)$ (defined in \cite{puterman2014markov}) is the bias of the optimistic MDP $P_e$ for which the optimistic policy $\pi_e$ obtains the maximum expected reward in the confidence interval.
Using the bound of the error arising from estimation of transition model in Lemma \ref{lemma: model_err_SA} into Eq. \eqref{eq: bellman h lemma result 0}, we obtain the following result with probability $1-1/t_e ^6$:
    \begin{align}
        & B^{\pi_e, {P}_e}(s,a) \nonumber \\
        & \leq 
        \frac{7 S (1+Ce)\log(S^2 A T)+ SCe \log{(eS)}}{\max\{1,N_{e}(s,a)\}} \|\tilde{h}(\cdot)\|_{\infty}, \label{eq: bellman h lemma result}
    \end{align}

Consequently, it is necessary to bound the model bias error term in Eq. \eqref{eq: bellman h lemma result}. We recall that the optimistic MDP $P_e$ corresponding to epoch $e$ maximizes the average reward in the confidence set as described in the formulation of Eq. \eqref{eq: optimistic opt 0}. In other words, $\Gamma_{\pi_e}^{P_e} \geq \Gamma_{\pi_e}^{P'}$ for every $P'$ in the confidence interval. Specifically, the model bias for optimistic MDP can be explicitly defined in terms of Bellman equation \cite{puterman2014markov}, as follows: 
    \begin{align}
        \tilde{h}(s) & \triangleq r_{\pi_e}(s) - \lambda_{\pi_e}^{P_e} + \Big(P_{e, \pi_e}(\cdot | s) \Big)^{T}\tilde{h}(\cdot), \label{eq: model bias main}
    \end{align}
    where the following simplified notations are used: $r_{\pi_e}(s) = \sum_{a \in \mathcal{A}} \pi_e (a|s) r(s,a)$, $P_{e, \pi_e}(s'| s) = \sum_{a \in \mathcal{A}} \pi_e (a|s) P_e (s'|s,a)$. Next, we incorporate the following result from \cite{agarwal2021concave} stated as Lemma 11 in Appendix F correlating model bias $\tilde{h}(\cdot)$ to the mixing time $T_{\texttt{mix}}$ of MDP $\mathcal{M}$.
\begin{align}
        \tilde{h}(s) - \tilde{h}(s') \leq T_{\texttt{mix}}, \forall s,s' \in \mathcal{S}.
    \end{align}
Using bound of bias term $\tilde{h}(\cdot)$ into Eq. \eqref{eq: bellman h lemma result}, we get the final bound for Bellman error with probability $1- 1/T^6$ as:
\begin{equation}
\small
         B^{\pi_e, {P}_e}(s,a) \leq \frac{ 7 S (1+Ce)\log(S^2 A T)+ SCe \log{(eS)}}{\max\{1,N_{e}(s,a)\}}  T_{\texttt{mix}}. \label{eq: bellman final bound}
\end{equation}

Note that Q-UCRL obtains a quadratic improvement for the Bellman error over the classical UC-UCRL \cite{agarwal2021concave}. We next present the final regret bound of Q-UCRL.
\begin{theorem} \label{thm: main_regr} In an unknown MDP environment $\mathcal{M} \triangleq (\mathcal{S}, \mathcal{A}, P, r, D)$, the  regret incurred by Q-URL (Algorithm \ref{algo: QUCRL}) across $T$ rounds is bounded as follows:
\begin{equation}
\small
    \mathcal{R}_{[0, T-1]} = \mathcal{O} \Bigg(\textstyle{{ \! S^5 A^4 T_{\texttt{mix}} {\log^3\bigg(\!\frac{T}{SA}\!\bigg) \log^{\frac{1}{2}}(T \sqrt{S}) \log(S^2AT)}} }\!\!\Bigg).
\end{equation}   
\end{theorem}

\begin{proof}
Please refer to Appendix G.
\end{proof}

\begin{remark}
    We reiterate that our characterization of regret incurred by Q-UCRL in Theorem \ref{thm: main_regr} is a martingale-free analysis and thus deviates from its classical counterparts \cite{fruit2018efficient,auer2008near,agarwal2021concave}. 
\end{remark}
\section{Conclusions}\label{sec: conclusion}
This paper demonstrates that quantum computing helps provide significant reduction in the regret bounds for infinite horizon reinforcement learning with average rewards. This paper not only unveils the quantum potential for bolstering reinforcement learning but also beckons further exploration into novel avenues that bridge quantum and classical paradigms, thereby advancing the field to new horizons.

A promising future direction is parameterized quantum RL, which has recently shown speedups for discounted settings \cite{xu2025accelerating}; achieving similar gains for average-reward problems, where the state-of-the-art in non-quantum RL is given by \cite{ganesh2025sharper}, remains an open challenge.


\section*{Impact Statement}
This paper presents work whose goal is to advance the field of Machine Learning. There are many potential societal consequences of our work, none which we feel must be specifically highlighted here.

\bibliography{main}
\bibliographystyle{icml2025}

\newpage
\appendix
\onecolumn


\section{Appendix A: Convexity of Optimization Problem}
In order to solve the optimistic optimization problem in Eq. \eqref{eq: optimistic opt 0} - \eqref{eq: optimistic opt -1} for epoch $e$, we adopt the approach proposed in \cite{agarwal2021concave,rosenberg2019online}. The key is that the constraint in \eqref{eq:mainnon} seems non-convex. However, that can be made convex with an inequality as
\begin{equation} \sum_{a \in \mathcal{A}} \rho(s',a)  \le \sum_{(s,a) \in \mathcal{S} \times \mathcal{A}} {P}_{e}(s'| s,a) \rho(s,a), ~\forall s' \in \mathcal{S} \label{eq:issue2}
\end{equation}
This makes the problem convex since $xy\ge c$ is convex region in the first quadrant. 

However, since the constraint region is made bigger, we need to show that the optimal solution satisfies \eqref{eq:mainnon} with equality. This follows since

\begin{equation}
   \sum_{s'\in \mathcal{S}} \sum_{(s,a) \in \mathcal{S} \times \mathcal{A}} {P}_{e}(s'| s,a) \rho(s,a) = 1 = \sum_{s'\in \mathcal{S}} \sum_{a \in \mathcal{A}} \rho(s',a)
\end{equation}

Thus, since we have point-wise inequality in \eqref{eq:issue2} with the sum of the two sides being the same, thus, the equality holds for all $s'$.

\newpage

\section{Appendix B: A brief overview of Quantum Amplitude Amplification}

Quantum amplitude amplification stands as a critical outcome within the realm of quantum mean estimation. This pivotal concept empowers the augmentation of the amplitude associated with a targeted state, all the while suppressing the amplitudes linked to less desirable states. The pivotal operator, denoted as $Q$, embodies the essence of amplitude amplification: $Q = 2|\psi\rangle\langle\psi| - I$, where $|\psi\rangle$ signifies the desired state and $I$ represents the identity matrix. This operator orchestrates a double reflection of the state $|\psi\rangle$ - initially about the origin and subsequently around the hyperplane perpendicular to $|\psi\rangle$ - culminating in a significant augmentation of the amplitude of $|\psi\rangle$.

Moreover, the application of this operator can be iterated to achieve a repeated amplification of the desired state's amplitude, effectively minimizing the amplitudes of undesired states. Upon applying the amplitude amplification operator a total of $t$ times, the outcome is $Q^t$, which duly enhances the amplitude of the desired state by a scaling factor of $\sqrt{N}/M$, where $N$ corresponds to the overall count of states and $M$ signifies the number of solutions.

The implications of quantum amplitude amplification extend across a diverse spectrum of applications within quantum algorithms. By bolstering their efficacy and hastening the resolution of intricate problems, quantum amplitude amplification carves a substantial niche. Noteworthy applications encompass:

\begin{enumerate}

\item  Quantum Search: In the quantum algorithm for unstructured search, known as Grover's algorithm \cite{grover1996fast}, the technique of amplitude amplification is employed to enhance the amplitude of the target state and decrease the amplitude of the non-target states. As a result, this technique provides a quadratic improvement over classical search algorithms.

\item Quantum Optimization: Quantum amplification can be used in quantum optimization algorithms, such as Quantum Approximate Optimization Algorithm (QAOA) \cite{farhi2014quantum}, to amplify the amplitude of the optimal solution and reduce the amplitude of sub-optimal solutions. This can lead to an exponential speedup in solving combinatorial optimization problems.

\item  Quantum Simulation: Quantum amplification has also been used in quantum simulation algorithms, such as Quantum Phase Estimation (QPE) \cite{kitaev1995quantum}, to amplify the amplitude of the eigenstate of interest and reduce the amplitude of other states. This can lead to an efficient simulation of quantum systems, which is an intractable problem in the classical domain.

\item  Quantum Machine Learning: The utilization of quantum amplification is also evident in quantum machine learning algorithms, including Quantum Support Vector Machines (QSVM) \cite{lloyd2013quantum}. The primary objective is to amplify the amplitude of support vectors while decreasing the amplitude of non-support vectors, which may result in an exponential improvement in resolving particular classification problems.
\end{enumerate}

In the domain of quantum Monte Carlo methods, implementing quantum amplitude amplification can enhance the algorithm's efficiency by decreasing the amount of samples needed to compute the integral or solve the optimization problem. By amplifying the amplitude of the desired state, a substantial reduction in the number of required samples can be achieved while still maintaining a certain level of accuracy, leading to significant improvements in efficiency when compared to classical methods. This technique has been particularly useful in \cite{hamoudi2021quantum} for achieving efficient convergence rates in quantum Monte Carlo simulations.
\newpage
\section{Appendix C: Proof of Lemma 3}
\begin{lemma} [Lemma 3 in the main paper] Execution of Q-UCRL (Algorithm \ref{algo: QUCRL}) up to the beginning of epoch $e+1$ with total $e$ completed epochs comprising $t = 1, 2, \ldots, t_e$ rounds,  guarantees that the following holds for $\Hat{P}_{e+1}$:
\begin{align} \label{eq: P hat bound2}
    \mathbb{P}[|\hat{P}_{e+1}(s'|s, a) - P(s'|s,a)| \leq \frac{eC \log(e{S}/\delta)}{N_{e+1}(s,a)}] \geq 1 - \delta
\end{align}
$\forall (s,a,s') \in \mathcal{S} \times \mathcal{A} \times \mathcal{S}$, where $C ={c\log^{1/2}(T\sqrt{S})}$ for some $c \in \mathbb{R}$ and $\{ \nu_{e}(s,a) \}$ are the state-action pair visitation counts up in epoch $e$, as depicted in Line 12 of Algorithm \ref{algo: QUCRL}, and $\{P(s'|s,a), \hat{P}_{e+1}(s'|s,a) \}$ are the actual and weighted average of transition probabilities by Eq. \eqref{eq: p hat formula}.

\begin{proof}
    For all $(s,a,s') \in \mathcal{S} \times \mathcal{A} \times \mathcal{S}$ , $\hat{P}_{e+1}(s'|s,a)$ could be equivalently expressed as:

    \begin{align} \label{eq: equivalent form for hat P}
        \hat{P}_{e+1}(s'|s,a) = \frac{\sum_{i=1}^e\nu_i(s,a)\tilde{P}_i(s'|s,a)}{\sum_{j=1}^e\nu_j(s,a)}
    \end{align}
    Denote event $|\tilde{P}_i(s'|s,a)-P(s'|s,a)| \leq {\frac{C \log(S/\delta)}{\nu_e(s,a)}}$ as event $E_i$. Thus if events $E_1,..., E_{e}$ all occur, we would have:
    \begin{align}
         &  |\hat{P}_{e+1}(s'|s, a) - P(s'|s,a)|\\
        & = \Bigg |\frac{\sum_{i=1}^e\nu_i(s,a)(\tilde{P}_i(s'|s,a) - P(s'|s,a))}{\sum_{j=1}^e\nu_j(s,a)}\Bigg |, \\
        & \leq \frac{\sum_{i=1}^e\nu_i(s,a)|\tilde{P}_i(s'|s,a) - P(s'|s,a)|}{\sum_{j=1}^e\nu_j(s,a)}, \label{eq: lemma3.1}\\
        & \leq \frac{\sum_{i=1}^e C \log({S}/\delta)}{\sum_{j=1}^e\nu_j(s,a)}, \label{eq: lemma3.2}\\
        & = \frac{eC \log({S}/\delta)}{\sum_{j=1}^e\nu_j(s,a)}, \\
        & = \frac{eC \log({S}/\delta)}{N_{e+1}(s,a)}, \label{eq: lemma3.3}
        \end{align}
    Where Eq. \eqref{eq: lemma3.1} is the result of triangle inequality and Eq. \eqref{eq: lemma3.2} is by Lemma \ref{lemma: P tilde bound}. The probability of events $E_1,..., E_{e}$ all occurring is:
    \begin{align}
         &  \mathbb{P}[E_1 \cap ...\cap E_{e}],\\
        & = 1-\mathbb{P}[E_1^C \cup ...\cup E_{e}^C],\\
        & \geq 1-\sum_{i=1}^{e}\mathbb{P}[E_i^C],\\
        & \geq 1-e\delta  \label{eq: lemma3.4}
        \end{align}
    Where Eq. \eqref{eq: lemma3.4} is by Lemma \ref{lemma: P tilde bound}. By combing Eq. \eqref{eq: lemma3.3} and Eq. \eqref{eq: lemma3.4} and substituting $e\delta$ with $\delta$, we would obtain Eq. \eqref{eq: P hat bound2}
\end{proof}
\end{lemma}
\newpage
\section{Appendix D: Proof of Lemma 4}
\begin{lemma} [Lemma 4 in the main paper]  Execution of Q-UCRL (Algorithm \ref{algo: QUCRL}) up to total $e$ epochs comprising $t = 1, 2, \ldots, t_e$ rounds,  guarantees that the following holds for confidence parameter $\delta_e = \frac{1}{S A t_e^7}$ with probability at least $1-1/t_{e}^6$:
\begin{align}
    &\|{P}_{e}(\cdot | s,a ) - {P}(\cdot | s,a ) \|_1, \nonumber\\
    &\leq  \frac{7 S (1+Ce)\log(S^2 A T)+ SCelog{(eS)}}{\max\{1,N_{e}(s,a)\}}, \label{eq: model_est_err2}
\end{align}
$\forall (s,a) \in \mathcal{S} \times \mathcal{A}$, where $\{ N_{e-1}^{(s,a)} \}$ is the cumulative state-action pair visitation counts up to epoch $e$, as depicted in Line 17 of Algorithm \ref{algo: QUCRL}, and $\{P(\cdot|s,a), P_e(\cdot|s,a) \}$ are the actual and estimated transition probabilities.
\end{lemma}
\begin{proof}
To prove the claim in Eq. \eqref{eq: model_est_err2}, note that for an arbitrary pair $(s,a)$ we have:
\begin{eqnarray}
	&& \|{P}_{e}(\cdot | s,a ) - {P}(\cdot | s,a ) \|_1 \nonumber \\
    & \leq& \|{P}_{e}(\cdot | s,a ) - \hat{P}_e(\cdot | s,a ) \|_1  \nonumber\\
    &&+ \|\hat{P}_{e}(\cdot | s,a ) - {P}(\cdot | s,a ) \|_1, \label{eq: model_est_err_0} \\
    & \leq& \frac{7 S \log(S^2 A t_e)}{\max\{1,N_{e}(s,a)\}} + \|\hat{P}_{e}(\cdot | s,a ) - {P}(\cdot | s,a ) \|_1, \label{eq: model_est_err_1}
    \end{eqnarray}
   \begin{eqnarray} 
 & =& \frac{7S\log(S^2 A t_e )}{\max\{1,N_{e}(s,a)\}} \nonumber\\&&+ \sum_{s' \in \mathcal{S}} |\underbrace{\hat{P}_{e}(s' | s,a ) - {P}(s' | s,a )}_{\text{(a)}}|, \label{eq: model_est_err_2}
\end{eqnarray}
where Eq. \eqref{eq: model_est_err_1} is a direct consequence of constraint Eq. \eqref{eq: optimistic opt -1} of optimization problem Eq. \eqref{eq: optimistic opt 0}. Next, in order to analyze (a) in Eq. \eqref{eq: model_est_err_2}, consider the following event:

\begin{align}
    \mathcal{E} &= \Bigg\{|{ \hat{P}_{e}(s' | s,a ) - {P}(s' | s,a )}| \nonumber \\
    &\leq  \frac{7\log(S^2 A t_e)}{\max\{1,N_{e}(s,a)\}}, \forall (s,a,s') \in \mathcal{S} \times \mathcal{A} \times \mathcal{S}\Bigg\},
\end{align}
For a tuple $(s,a,s')$, by replacing $\delta = e^{-\epsilon \cdot N_e(s,a)}$ into Eq. \eqref{eq: Subgaussian} from Lemma \ref{lem: SubGau}, we get:
\begin{align}
    & \mathbb{P}\left[|{\hat{P}_{e}(s' | s,a ) - {P}(s' | s,a )}| \leq \frac{Ce\log(e{S})}{N_e{(s,a)}} +Ce\epsilon \right]  \nonumber \\
    & \quad \quad \geq 1-e^{-\epsilon \cdot N_{e}(s,a)},  \label{eq: model_est_err_new_1}
\end{align}
By plugging in $\epsilon = \frac{\log(S^2 A t_{e}^7)}{N_e(s,a)}$ into Eq. \eqref{eq: model_est_err_new_1}, we get:
\begin{align}
    &\mathbb{P}\left[|{\hat{P}_{e}(s' | s,a ) - {P}(s' | s,a )}| \geq \frac{Ce\log  {(eS)}+Ce\log(S^2 A t_e^7)}{N_e(s,a)} \right] \nonumber \\
    & \quad \quad \leq \frac{1}{S^2 A t_e^7},  \label{eq: model_est_err_new_2}
\end{align}
Finally, we sum over all possible values of $N_e(s,a)$ as well as tuples $(s,a,s')$ to bound the probability in the RHS of \eqref{eq: model_est_err_new_2} that implies failure of event $\mathcal{E}$:
\begin{align}
    \sum_{(s,a,s') \in \mathcal{S} \times \mathcal{A} \times \mathcal{S}}\sum_{N_e(s,a) = 1}^{t_e} \frac{1}{S^2 A t_e^7} \leq \frac{1}{t_e^6} \label{eq: model_est_err_new_3}
\end{align}
Here, it is critical to highlight that we obtain this high probability bound for model errors by careful utilization of Lemma \ref{lem: SubGau} and quantum signals $|\bf{\psi}\rangle^t$ via Eq. \eqref{eq: model_est_err2} - \eqref{eq: model_est_err_new_3}, while avoiding classical concentration bounds for probability norms that use martingale style analysis such as in the work \cite{weissman2003inequalities}. 
 Finally, plugging the bound of term (a) as described by event $\mathcal{E}$, we obtain the following with probability $1-1/t_e^6$:
\begin{align}
 & \|{P}_{e}(\cdot | s,a ) - {P}(\cdot | s,a ) \|_1 \nonumber \\
 & \leq \frac{7 S \log(S^2 A t_e)}{\max\{1,N_{e}(s,a)\}} + \sum_{s' \in \mathcal{S}} \frac{ Ce\log {(eS)} + Ce \log(S^2 A t_e^7)}{\max\{1,N_{e}(s,a)\}}, \\
 & \leq \frac{7 S \log(S^2 A t_e)}{\max\{1,N_{e}(s,a)\}} + \sum_{s' \in \mathcal{S}} \frac{ Ce\log {(eS)} + 7Ce \log(S^2 A t_e)}{\max\{1,N_{e}(s,a)\}}, \\
 & = \frac{7 S (1+Ce)\log(S^2 A t_e)+ SCelog{(eS)}}{\max\{1,N_{e}(s,a)\}}, \\
 & \leq \frac{7 S (1+Ce)\log(S^2 A T)+ SCelog{(eS)}}{\max\{1,N_{e}(s,a)\}}.
\end{align}
This proves the claim as in the statement of the Lemma.
\end{proof}
\newpage
\section{Appendix E: Proof of Lemma 5}
\begin{lemma}
    [Regret Decomposition, Same as Lemma 5 in the main paper] Regret of Q-UCRL defined in Eq. \eqref{eq: regret defn} can be decomposed as follows:
    \begin{align}
        \mathcal{R}_{[0, T-1]} \leq \sum_{e = 1}^{E} T_{e} \Big(\Gamma_{\pi_{e}}^{\tilde{P}_{e}} - \Gamma_{\pi_{e}}^{{P}} \Big).
    \end{align}
\end{lemma}
\begin{proof}
\begin{align}
    & \mathcal{R}_{[0, T-1]} =  T \times \Gamma_{\pi^{*}}^{P} - \mathbb{E} \Big[\sum_{t=0}^{T-1} r(s_t, a_t) \Big], \label{eq: regr decomp 0}\\
    & = T \times \Gamma_{\pi^{*}}^{P} - \sum_{e = 1}^{E} T_{e} \Gamma_{\pi^{*}_{e}}^{P} \nonumber  + \sum_{e = 1}^{E} T_{e} \Gamma_{\pi^{*}_{e}}^{P} - \mathbb{E} \Big[ \sum_{t=0}^{T-1} r(s_t, a_t)\Big], \\
    & = \sum_{e = 1}^{E} T_{e} \big(\underbrace{\Gamma_{\pi^{*}}^{P} - \Gamma_{\pi^{*}_{e}}^{P}}_{\text{(a)}} \big) +  \sum_{e = 1}^{E} T_{e} \Gamma_{\pi^{*}_{e}}^{P} - \mathbb{E} \Big[ \sum_{t=0}^{T-1} r(s_t, a_t)\Big], \label{eq: regr decomp 1}\\
    & = \sum_{e = 1}^{E} T_{e} \Gamma_{\pi^{*}_{e}}^{P} - \mathbb{E} \Big[ \sum_{t=0}^{T-1} r(s_t, a_t)\Big], \\
    & \leq  \sum_{e = 1}^{E} T_{e} \Gamma_{\pi_{e}}^{{P}_{e}} - \mathbb{E} \Big[ \sum_{t=0}^{T-1} r(s_t, a_t)\Big], \label{eq: regr decomp 2} \\
    & =  \sum_{e = 1}^{E} T_{e} \Gamma_{\pi_{e}}^{{P}_{e}} -  \sum_{e = 1}^{E} T_{e} \Gamma_{\pi_{e}}^{{P}} +  \sum_{e = 1}^{E} T_{e} \Gamma_{\pi_{e}}^{{P}} - \mathbb{E} \Big[\sum_{t=0}^{T-1} r(s_t, a_t) \Big], \\
    & =  \sum_{e = 1}^{E} T_{e} \Gamma_{\pi_{e}}^{{P}_{e}} -  \sum_{e = 1}^{E} T_{e} \Gamma_{\pi_{e}}^{{P}} +   \Big[\sum_{e = 1}^{E} \sum_{t = t_{e-1} + 1}^{t_{e}}\underbrace{\Gamma_{\pi_{e}}^{{P}} - \mathbb{E}[r(s_t, a_t)]}_{\text(b)} \Big], \label{eq: regr decomp 3} \\
    & =  \sum_{e = 1}^{E} T_{e} \Big(\Gamma_{\pi_{e}}^{\tilde{P}_{e}} - \Gamma_{\pi_{e}}^{{P}} \Big).  \label{eq: regr decomp -1}
\end{align}
\end{proof}

Firstly, (a) in Eq. \eqref{eq: regr decomp 1} is 0 owing to the fact that the agent would have gotten $\pi_{e}^{*}$ exactly as the true $\pi^{*}$ in epoch $e$ solving for original problem Eq. \eqref{eq: opt_obj}, had it known the true transition model $P$. Next, Eq. \eqref{eq: regr decomp 2} is because $\{P_e, \pi_e \}$ are outputs from the optimistic problem Eq. \eqref{eq: optimistic opt 0} solved at round $e$, which extends the search space of Eq. \eqref{eq: opt_obj} and therefore always upper bounds true long-term rewards. Finally, term (b) in Eq. \eqref{eq: regr decomp 3} is 0 in expectation, because the RL agent actually collects rewards by playing policy $\pi_e$ against the true $P$ during epoch $e$. 

\newpage
\section{Appendix F: Some Auxiliary Lemmas used in this Paper}

\begin{lemma} [Lemma 5.3, \cite{agarwal2021concave}]    The Bellman errors for any arbitrary state-action pair $(s,a)$ corresponding to epoch $e$ is bounded as follows:
\begin{align}
    B^{\pi_e, {P}_e}(s,a) \leq \Big\|{P}_e(\cdot|s,a) - {P}(\cdot|s,a) \Big\|_{1} \|\tilde{h}(\cdot) \|_{\infty},  
\end{align}
where $\tilde{h}(\cdot)$ (defined in \cite{puterman2014markov}) is the bias of the optimistic MDP $P_e$ for which the optimistic policy $\pi_e$ obtains the maximum expected reward in the confidence interval.
\end{lemma}
\begin{lemma} \label{lemma: h tmix result} [Lemma D.5, \cite{agarwal2021concave}]
For an MDP with the transition model $P_e$ that generates rewards $r(s,a)$ on playing policy $\pi_e$, the difference of biases for states $s$ and $s'$ are bounded as:  
    \begin{align}
        \tilde{h}(s) - \tilde{h}(s') \leq T_{\texttt{mix}}, \forall s,s' \in \mathcal{S}.
    \end{align}
\end{lemma}
\newpage
\section{Appendix G: Proof of Theorem 1}
\begin{theorem} [Theorem 1 in main paper] In an unknown MDP environment $\mathcal{M} \triangleq (\mathcal{S}, \mathcal{A}, P, r, D)$, the  regret incurred by Q-URL (Algorithm \ref{algo: QUCRL}) across $T$ rounds is bounded as follows:
\begin{equation}
    \mathcal{R}_{[0, T-1]} = \mathcal{O} \Bigg({ S^5 A^4 T_{\texttt{mix}} {\log^3\bigg(\frac{T}{SA}\bigg) \log^{1/2}(T \sqrt{S}) \log(S^2AT)}} \Bigg).
\end{equation}   
\end{theorem}
\begin{proof}
Using the final expression from regret decomposition, i.e., Eq. \eqref{eq: regr decomp -1}, we have: 
    \begin{align}
         & \mathcal{R}_{[0, T-1]} = \sum_{e = 1}^{E} T_{e} \Big(\Gamma_{\pi_{e}}^{\tilde{P}_{e}} - \Gamma_{\pi_{e}}^{{P}} \Big), \\
         &= \frac{1}{T} \sum_{e = 1}^{E} \sum_{t = t_{e-1}+1}^{t_{e}} \sum_{s,a} \rho_{\pi_{e}}^{P} B^{\pi_{e}, {P}_{e}} (s,a), \label{eq: final regr 0} \\
        & \leq  \sum_{e = 1}^{E}\sum_{t = t_{e-1} + 1}^{t_{e}} \frac{ 7 S (1+Ce)\log(S^2 A T)+ SCelog{(eS)}}{\max\{1,N_{e}(s_t,a_t)\}}  T_{\texttt{mix}}
        , \label{eq: final regr 1}\\
        & =  \sum_{e = 1}^{E}\sum_{t = t_{e-1} + 1}^{t_{e}} \sum_{s,a} \Big(\mathbf{1}[s_t = s, a_t = a] \cdot \nonumber \\
        & \hspace{1cm} \frac{7 S (1+Ce)\log(S^2 A T)+ SCelog{(eS)}}{\max\{1,N_{e}(s,a)\}}  T_{\texttt{mix}} \Big), \\ 
        & =  \sum_{e = 1}^{E} \sum_{s,a} \nu_{e}(s,a)\cdot \nonumber \\
        & \hspace{1cm} \frac{7 S (1+Ce)\log(S^2 A T)+ SCelog{(eS)}}{\max\{1,N_{e}(s,a)\}}  T_{\texttt{mix}}, \label{eq: final regr 2}\\
        &\hspace{-5mm} = \sum_{s,a} T_{\texttt{mix}} (7S(E+CE^2)\log(S^2 AT)+SCE^2\log {(ES)}) \nonumber \\
        & \cdot \sum_{e = 1}^{E} \frac{\nu_{e}(s,a)}{\max \{1,N_{e} (s,a)\}}, \label{eq: final regr 3}
        \end{align}
        \begin{align}
        & \leq \sum_{s,a} T_{\texttt{mix}} (7S(E+CE^2)\log(S^2 AT)+SCE^2\log {(ES)}) E, \nonumber \\
        & = T_{\texttt{mix}} E(7S^2A(E+CE^2)\log(S^2 AT)+S^2ACE^2\log {(ES)}) \nonumber , \\
        & = \mathcal{O} \Bigg({ S^5 A^4 T_{\texttt{mix}} {\log^3\bigg(\frac{T}{SA}\bigg) \log^{1/2}(T \sqrt{S}) \log(S^2AT)}} \Bigg). \label{eq: final regr -1}
    \end{align}
where in Eq. \eqref{eq: final regr 0}, we directly incorporate the result of Lemma \ref{lemma: regr bellman}. Next, we obtain Eq. \eqref{eq: final regr 1} by using the Bellman error bound in Eq. \eqref{eq: bellman final bound}, as well as unit sum property of occupancy measures \cite{puterman2014markov}. Furthermore, in Eq. \eqref{eq: final regr 1} note that we have omitted the regret contribution of \textit{bad event} (Remark \ref{remark: bad event}) which is atmost $\tilde{\mathcal{O}} \Bigg(\frac{2 S T_{\texttt{mix}}}{T^4} \Bigg)$ as obtained below:
\begin{align}
    \sum_{e = 1}^{E} & \sum_{t = t_{e-1} + 1}^{t_{e}} \frac{2 S T_{\texttt{mix}}}{t_e^6} \leq  \sum_{e = 1}^{E} \frac{2 S T_{\texttt{mix}}}{t_e^5}, \\
    & \leq \sum_{t_e = 1}^{T} \frac{2 S T_{\texttt{mix}}}{t_e^5} = \tilde{\mathcal{O}} \Bigg(\frac{2 S T_{\texttt{mix}}}{T^4} \Bigg). 
\end{align}

Subsequently, Eq. \eqref{eq: final regr 2} - \eqref{eq: final regr 3} are obtaining by using the definition of epoch-wise state-action pair visitation frequencies $\nu_e(s,a)$ and the epoch termination trigger condition in Line 13 of Algorithm \ref{algo: QUCRL}. Finally, we obtain Eq. \eqref{eq: final regr -1} by the using Proposition 18 in Appendix C.2 of \cite{auer2008near}.
\end{proof}


\end{document}